%% file: main.tex
\documentclass[10pt,twocolumn,letterpaper]{article}

\usepackage{cvpr}              

\input{preamble}

\usepackage{bm}
\usepackage{soul}
\usepackage{multirow}
\usepackage[utf8]{inputenc}
\usepackage{graphicx}
\usepackage{amsmath}
\usepackage{amsthm}
\usepackage{amssymb}
\usepackage{booktabs}
\usepackage{algorithm}
\usepackage{amsfonts}
\usepackage{algpseudocode}
\usepackage{titling}

\definecolor{cvprblue}{rgb}{0.21,0.49,0.74}
\usepackage[pagebackref,breaklinks,colorlinks,citecolor=cvprblue]{hyperref}

\def\paperID{11831} 
\def\confName{CVPR}
\def\confYear{2024}

\title{Noise Adaptor: Enhancing Low-Latency Spiking Neural Networks through Noise-Injected Low-Bit ANN Conversion}

\author{Chen Li\\
King’s College London\\
London, UK\\
{\tt\small chen.7.li@kcl.ac.uk}
\and
Bipin.Rajendran\\
King’s College London\\
London, UK\\
{\tt\small bipin.rajendran@kcl.ac.uk}
}

\begin{document}
\maketitle
\input{sec/0}    
\input{sec/1}

\input{sec/2}

\input{sec/3}

\input{sec/4}

\input{sec/5}
\input{sec/6}
{
    \small
    \bibliographystyle{ieeenat_fullname}
    \bibliography{main}
}

\input{sec/suppl}

\end{document}

%% file: preamble.tex
\usepackage[dvipsnames]{xcolor}


%% file: sec/0.tex
\begin{abstract}
We present Noise Adaptor, a novel method for constructing competitive low-latency spiking neural networks (SNNs) by converting noise-injected, low-bit artificial neural networks (ANNs). This approach builds on existing ANN-to-SNN conversion techniques but offers several key improvements: (1) By injecting noise during quantized ANN training, Noise Adaptor better accounts for the dynamic differences between ANNs and SNNs, significantly enhancing SNN accuracy. (2) Unlike previous methods, Noise Adaptor does not require the application of run-time noise correction techniques in SNNs, thereby avoiding modifications to the spiking neuron model and control flow during inference. (3) Our method extends the capability of handling deeper architectures, achieving successful conversions of activation-quantized ResNet-101 and ResNet-152 to SNNs. We demonstrate the effectiveness of our method on  CIFAR-10 and ImageNet, achieving competitive performance. The code will be made available as open-source.

\end{abstract}

%% file: sec/1.tex
\section{Introduction}
\label{sec:intro}
In recent years, deep spiking neural networks (SNNs) have gained substantial attention due to their rapid, efficient information processing abilities which closely mimic those of biological neural systems \cite{diehl2015fast,rueckauer2017conversion,li2023unleashing,deng2021optimal,wu2018spatio,zhang2018plasticity,meng2022training,fang2021deep,li2021free}. As a biologically plausible alternative to artificial neural networks (ANNs), SNNs provide valuable insights into the functioning of the human brain and have become one of the focal points for researchers in the pursuit of developing next-generation artificial intelligence (AI) systems \cite{maass1997networks}. This growing interest is driven by the potential benefits SNNs can offer in terms of power consumption, latency, and other performance metrics \cite{roy2019towards}, which promise significant advancements in the field of AI.

To fully harness the potential of SNNs, there have been rapid strides in both hardware and algorithms. On the hardware front, remarkable progress has been made in the evolution and enhancements of neuromorphic systems, including, but not limited to SpiNNaker 2 \cite{mayr2019spinnaker}, BrainScaleS-2 \cite{pehle2022brainscales}, Loihi 2 \cite{orchard2021efficient}, and DYNAP-SE2 \cite{ccakal2023training}. These cutting-edge hardware prototypes feature enhanced chip capacity, innovative solutions for efficient spike routing, and breakthroughs in architectural design. 

Concurrently, algorithmic research has demonstrated that SNNs can achieve accuracy levels on par with ANNs through ANN-to-SNN conversion \cite{diehl2015fast, rueckauer2017conversion, sengupta2019going}. However, this often requires long simulation times, negating the benefits of using SNNs in the first place. As a result, research has shifted towards low-latency SNNs, also known as fast SNNs, which aim to achieve high inference accuracy with as few time steps as possible \cite{li2022quantization, li2023unleashing, hao2023bridging}. Given the heightened interest and fast progress in low-latency SNN algorithms, the latency required for SNNs to deliver satisfactory accuracy has consistently decreased from thousands of time steps \cite{rueckauer2017conversion, sengupta2019going} to hundreds \cite{deng2021optimal,li2021free} and, more recently, fewer than ten time steps \cite{li2021differentiable, fang2021deep}. 

A particularly promising approach, known as quant-ANN-to-SNN conversion, has proven highly effective in building low-latency SNNs \cite{li2022quantization, bu2021optimal, meng2022training}. This approach involves quantizing the activations of ANNs prior to their conversion into SNNs. Unlike other low-latency SNN learning algorithms that rely on surrogate gradients, quant-ANN-to-SNN conversion does not require unfolding the neural network in the temporal dimension \cite{neftci2019surrogate} during training. By avoiding this temporal unfolding step, the computational complexity and memory complexity are reduced from $\mathcal{O}(T)$ to $\mathcal{O}(1)$, where $T$ represents the number of time steps in an SNN. This significantly accelerates the training process and improves memory efficiency. As a result, quant-ANN-to-SNN conversion becomes an appealing option for rapidly developing low-latency SNNs, without the high training costs typically associated with other methods. The quant-ANN-to-SNN conversion algorithms have been widely applied to image classification \cite{li2022quantization}, object detection \cite{hu2023fast}, and natural language processing \cite{you2024spikezip}.

The construction of low-latency SNNs through conversion from activation-quantized ANNs often encounters a significant challenge: the presence of occasional noise, also known as unevenness error.  This issue stems from the variability in postsynaptic spike sequences that are not considered during the training of low-bit ANNs  \cite{li2022quantization, bu2021optimal,meng2022training}. This noise can lead to diminished accuracy in SNNs and increased time steps required to reach the desired accuracy level. To address this issue, various strategies have been proposed, focusing primarily on integrating artificial noise correction mechanisms during SNN inference \cite{li2022quantization, hu2023fast, hao2023reducing, hao2023bridging}. Despite these efforts, a fully comprehensive and efficient solution for seamless, end-to-end conversion from low-bit ANNs to low-latency SNNs, while effectively mitigating this challenge, remains elusive. Furthermore, previous attempts to construct low-latency SNNs from quantized ANNs have often struggled to scale the method to deep learning models deeper than ResNet-50.

In this study, we introduce \textbf{Noise Adaptor}, a novel method to mitigate both of these challenges in SNNs developed via quant-ANN-to-SNN conversion. Our technical contributions are threefold:

\begin{itemize}

\item We establish that occasional noise, a primary concern in current quant-ANN-to-SNN studies, can be effectively modeled and managed through a learning-based approach. Our research indicates that injecting controllable noise during training can better capture spike dynamics, thereby improving the tolerance to occasional noise.

\item We address the scalability issue of quant-ANN-to-SNN conversion. We reported the results of converting low-bit ANN with up to 101 and 152 layers, achieving competitive SNN performance.

\item Our proposed method is agnostic to the type of pooling layer used in the ANN. This allows us to optimize and convert ANNs with max-pooling to SNNs with average-pooling while maintaining competitive performance.

\end{itemize}

%% file: sec/2.tex
\section{Related Work}
\paragraph{Fast SNNs.}
As the demand for real-time, energy-efficient computing continues to grow, fast SNNs have come to the forefront of contemporary SNN research. The leading techniques for developing fast SNNs are quant-ANN-to-SNN conversion \cite{li2022quantization, bu2021optimal,meng2022training}
and surrogate gradients \cite{neftci2019surrogate,shrestha2018slayer,wu2018spatio,fang2021deep}, and our research falls into the former catagory. The effectiveness of the quantized ANN-to-SNN conversion technique lies in the fact that, when an ANN is trained with quantized activations to mimic discrete spike counts, it becomes more amenable to conversion into an SNN compared to those trained without such constraints \cite{li2022quantization, bu2021optimal,meng2022training}. A notable challenge associated with the quant-ANN-to-SNN conversion is that the reduction in latency amplifies the impact of occasional noise \cite{li2022quantization}, which in turn hinders the achievement of high accuracy within a limited number of time steps. Several methods have been proposed to address this issue. QCFS \cite{bu2021optimal} provides a theoretical analysis of the occasional noise and amortizes this noise by extending the simulation duration. QFFS \cite{li2022quantization} detects occasional noise at the runtime of SNN and generates negative spikes to counteract it. More recently, SRP \cite{hao2023reducing} and COS \cite{hao2023bridging} have employed a two-stage approach for noise reduction: initially running an SNN for multiple time steps to identify occasional noise, followed by a second run to correct it.

\begin{figure}
\centering
\includegraphics[width=0.465\textwidth]{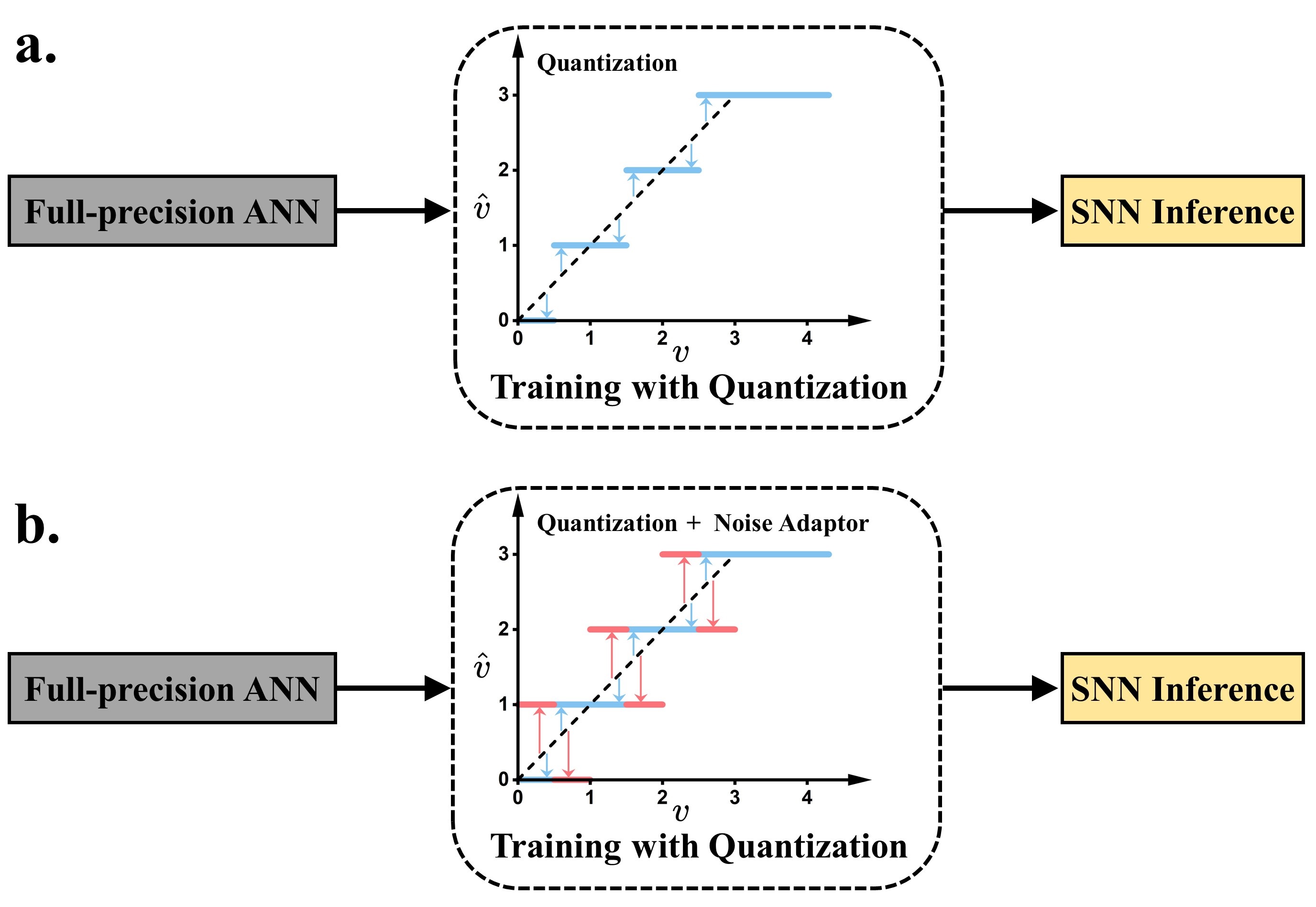} 

\caption{ \textbf{a.} In the standard quant-ANN-to-SNN conversion process, ANNs are trained with activation quantization to model the discrete spikes characteristic of SNNs. \textbf{b.} Our method goes further by introducing controllable noise into the activation quantization, optimizing the training process for both quantization accuracy and noise resilience.}

\label{fig1}
\end{figure}

\paragraph{Noise in SNNs.}
A fundamental trait of SNNs is the existence of noise \cite{kempter1998extracting,stein2005neuronal,faisal2008noise,ma2023exploiting}. Understanding the role and implications of noise is crucial for improving the robustness and functionality of SNNs. During the process of ANN-to-SNN conversion, SNNs are subject to multiple noise sources that can impair inference performance, such as rate-coding noise, sub-threshold noise, and supra-threshold noise \cite{diehl2015fast}. To counteract these effects, a range of strategies have been developed, including the use of analog input coding \cite{rueckauer2017conversion}, soft-reset (or reset-by-subtraction) \cite{han2020rmp,rueckauer2017conversion}, bias shift \cite{deng2021optimal}, and pre-charging the membrane potential \cite{hwang2021low}. For applications aiming to create low-latency or ultra-low-latency SNNs by ANN-to-SNN conversion, it is necessary to address and minimize occasional noise to achieve the desired accuracy \cite{li2022quantization}. Moreover, when deploying SNNs on neuromorphic hardware \cite{schemmel2010wafer}, especially on those utilizing noisy components like memristors \cite{burr2017neuromorphic} and skyrmions \cite{chen2020nanoscale} for weight storage, it is critical to account for noise. Doing so helps prevent system malfunctions and significant accuracy losses, ensuring the optimal functioning of spike-based computation systems.

%% file: sec/3.tex
\section{Preliminary and Background}
In this section, we provide an overview of the foundational concepts relevant to our work. We first describe the spiking neuron and synaptic models employed in SNNs. We then revisit the framework for converting the activation-quantized ANNs into SNNs.

\paragraph{Spiking Neuron and Synaptic Models.} We use the integrate-and-fire model for spiking dynamics in this paper.  To allow for more precise information processing within spiking neurons, the reset mechanism of the neuronal model is changed from reset-to-resting-potential to reset-by-subtraction, as discussed in previous works \cite{bu2021optimal,li2022quantization, rueckauer2017conversion,han2020rmp}. The model's dynamics are described by the following equations:
\begin{equation}\label{eq1}
\bm{u}_{t}^{l} = \bm{u}_{t-\Delta{t}}^{l}+\widetilde{\bm{W}}_t^l\bm{z}_{t}^{l-1}{th}^{l-1}+\widetilde{\bm{B}}^l-\bm{z}_{t-\Delta{t}}^{l}{th}^{l}, 
\end{equation}
\begin{equation}\label{eq2}
\bm{z}_{t}^{l} =\Theta(\bm{u}_{t}^{l}-{th}^{l}).
\end{equation} 
\noindent In Equation \eqref{eq1}, $\bm{u_}t^l$ and $\bm{u}_{t-\Delta{t}}^l$ represent the membrane potential of spiking neurons situated in layer $l$ at time $t$ and time $t-\Delta{t}$, respectively, with $\Delta{t}$ denoting the simulation time resolution. $\widetilde{\bm{W}}_t^l$ and $\widetilde{\bm{B}}^l$ are the synaptic weights and biases within layer $l$. The spikes generated by spiking neurons in layer $l$ at time $t$ are denoted by 
$\bm{z}_{t}^l$, which is 
calculated according to the membrane potential $\bm{u_}t^l$, the firing threshold ${th}^{l}$, and Heaviside step function $\Theta(\cdot)$, as shown in Equation \eqref{eq2}. To optimize GPU simulations, deep SNNs often simplify synaptic models, treating synaptic weights as constant and time-invariant \cite{bogdan2021towards,gammaitoni1998stochastic}, represented as
$\widetilde{\bm{W}}_t^l= \widetilde{\bm{W}}^l$, where \(\widetilde{\bm{W}}_t^l\) represents the synaptic weight at any given time \(t\), and \(\widetilde{\bm{W}}^l\) denotes the constant synaptic weight.

\paragraph{Quant-ANN-to-SNN Conversion.} The process of converting an activation-quantized artificial neural network into a spiking neural network is referred to as Quant-ANN-to-SNN conversion \cite{li2022quantization, bu2021optimal,meng2022training}. Quantization serves as a means of compressing information, thus decreasing the number of spikes required by the spiking neurons to represent ANN activations \cite{li2022quantization}. This technique has been widely reported to bring considerable savings in spike counts and SNN inference latency \cite{li2022quantization, bu2021optimal,meng2022training, hu2023fast,hao2023bridging,hao2023reducing}. This conversion works effectively because the quantization and clipping used in ANNs mimic key SNN characteristics, such as spike discreteness and the cut-off frequency (the upper limit firing rate for spiking neurons), thereby minimizing accuracy loss during conversion. The activation quantization in ANNs, exemplified here with ReLU, is formalized as
\begin{equation}\label{eq3}
	\hat{v}^l = s^l\left\lfloor clip \left( \frac{v^l}{s^l}, 0, p\right) \right\rceil.
\end{equation}
\noindent In this equation, \(v^l\) represents the original full-precision activation in layer \(l\), and \(\hat{v}^l\) denotes the quantized activation. The term $s^l$ is the quantization scale factor, and $clip(z, 0, p)$ truncates the elements of $z$ to the range $[0, p]$, where $p$ is a positive integer as the quantization upper bound\footnote{In ANN activation quantization, $p$ is typically derived using the formula $p=2^b-1$, with $b$ being the bit precision. Nevertheless, when it comes to quant-ANN-to-SNN conversion,  the value of $p$ is more flexible and it can be any positive integer. 
}. The rounding operation $\lfloor z \rceil$  rounds $z$ to the nearest integer. The parameters of the SNN are derived from the ANN parameters as follows:
\begin{equation}\label{eq4}
\left\{
\begin{aligned}
   \widetilde{\bm{W}}^1 &= \bm{W}^l, \\
   \widetilde{\bm{B}}^1 &= \bm{B}^l, \\
   th^l &= ps^l.
\end{aligned}
\right.
\end{equation}
\noindent Here, $\bm{W}^l$ and $\bm{B}^l$ represent the weight and bias of layer $l$ in the ANN, respectively. The spiking threshold ${th}^l$ is a function of the positive integer threshold $p$ and the quantization step size ${s^l}$. The integrate-and-fire mechanism in SNNs approximates the``flooring'' operation in ANN quantization, in contrast to the ``rounding'' operation used in Equation \ref{eq1}. This discrepancy is compensated by initializing the membrane potential at \(0.5 {th}^l\) across all hidden layers at the start of SNN simulation \cite{li2022quantization, bu2021optimal}.

%% file: sec/4.tex
\section{Noise Adaptor}
In this section, we elucidate our methodology for developing low-latency SNNs, which we refer to as Noise Adaptor.
We begin by outlining the formulation of Noise Adaptor, detailing the calculation involved in both the forward passes and backward gradients. We then delve into the design rationale behind our noise injection technique and describe its impact on the activation quantizer. Our method enhances the existing quant-ANN-to-SNN conversion process, improving several key aspects. These improvements, along with connections to related theoretical work, are discussed in the following sections.

\subsection{Formulation}
When applying Noise Adaptor during ANN training, the computations on the forward pass of artificial neurons are described below:
\begin{equation}\label{eq5}
	\hat{v}^l = s^l\left\lfloor clip\left( \frac{v^l}{s^l}+\epsilon, 0, p\right) \right\rceil,
\end{equation}
\begin{equation}\label{eq6}
    \epsilon \sim \mathcal{U}(-0.5, 0.5).
\end{equation}
\noindent Compared to the activation quantization function applied in the existing quant-ANN-to-SNN conversion methods, we have one additional parameter added to the pre-activation value. $\epsilon$ represents a uniformly distributed noise term over the interval $(-0.5, 0.5)$, as shown in Equation \eqref{eq6}. The dimensions of this noise term are identical to those of $v^l$. $\epsilon$ is independently sampled in the quantizer in each layer during forward pass. The sampled $\epsilon$ will also affect backward gradients according to the equation as follows:
\begin{equation}\label{eq7}
\frac{\partial{\hat{v}^l}}{\partial{v^l}} =
\begin{cases}
0								& \text{if $\frac{v^l}{s^l}+\epsilon \leqslant 0$} \\
1			& \text{if $0 < \frac{v^l}{s^l}+\epsilon  < p$} \\
0.							& \text{if $\frac{v^l}{s^l}+\epsilon  \geqslant p$} \\
\end{cases}
\end{equation}
\noindent This approach ensures that the influence of noise remains consistent throughout both forward pass and backward pass. This extra noise in gradients also facilitates exploring a broader range of the loss landscapes, potentially leading to the discovery of better local minima.
This backward gradient is calculated by applying the straight-through estimator to bypass non-differential computations \cite{bengio2013estimating}.
Our method applies the trainable quantization scale factor $s^l$ to let the network learn the optimal scale that minimizes the loss due to quantization. The gradient of $\frac{\partial{\hat{v}^l}}{\partial{s^l}}$ is calculated in the following manner:
\begin{equation}\label{eq8}
\frac{\partial{\hat{v}^l}}{\partial{s^l}} =
\begin{cases}
0									& \text{if $\frac{v^l}{s^l}+\epsilon \leqslant 0$} \\
-(\frac{v^l}{s^l}+\epsilon) + \lfloor \frac{v^l}{s^l} +\epsilon\rceil			& \text{if $0 < \frac{v^l}{s^l}+\epsilon < p$} \\
p.									& \text{if $\frac{v^l}{s^l}+\epsilon \geqslant p$} \\
\end{cases}
\end{equation}

\subsection{Training with Noise Injection }

After integrating the noise term into the activation quantization function of ANNs, the activation output $\hat{v}^l$ becomes stochastic and capable of transitioning to adjacent discrete states either upwards or downwards by $s^l$. This stochasticity mirrors the behavior of spiking neurons, which, under the influence of occasional noise, can either increase or decrease their spike count, typically by one \cite{li2022quantization,hao2023reducing}. It is possible to model larger spike jumps by injecting higher amplitude noise. However, we found that this can destabilize the training process, especially with intricate datasets such as ImageNet. Therefore, our approach only models the simplest situations where the spike count changes are limited to one. 

Another key consideration of our method is that the probability of state transitions in a spiking neuron should be related to its input values. For instance, a spiking neuron with an input of $0.99{th}^l$ (${th}^l$ is the firing threshold of spiking neurons in layer $l$) is more likely to produce an additional spike in response to occasional noise than one with an input of $0.01 {th}^l$. The closer the input is to the threshold of the next quantization level, the higher the chance of transitioning to a different integer state. To accurately model this effect, we add the noise term $\epsilon$ before the clipping and rounding steps in the ANN activation function. This allows the pre-activation amplitude to influence the likelihood of state changes.

Figure \ref{fig1} provides a comprehensive illustration of the impact of noise injection on the conversion from quantized ANN to SNN. For simplicity, we set the parameters to $p=3$ and $s=1$. In Figure \ref{fig1}.a, the standard conversion process is shown, where each variable $v$ is consistently rounded in a fixed direction, either up or down. In contrast, Figure \ref{fig1}.b shows that Noise Adaptor introduces stochasticity in activation, which creates additional possibilities by allowing values to be rounded to alternative states. This increased variability captures the spiking dynamics more effectively, leading to a smoother ANN-to-SNN conversion.

Our method exhibits an activation quantizer behavior similar to that of stochastic rounding, a technique employed in finite precision arithmetic and neural network quantization algorithms. The key differences between our method and stochastic rounding are outlined below:  (1). Objective Differences. Stochastic rounding is designed to minimize the rounding error that accumulates in neural networks. Instead, our proposed noise injection strategy is specifically tailored to emulate the stochastic behavior of spiking neurons, where noise can cause fluctuations in spike counts. (2). Format Differences. Stochastic rounding requires a fixed noise type and amplitude. However, Noise Adaptor offers flexibility in both the form and amplitude of the injected noise. For instance, it can inject different types of noise, such as Gaussian noise, or adjust the noise amplitude, like doubling it to \( \epsilon \sim \mathcal{U}(-0.5, 0.5) \), as long as it can model the spike jumping behavior while maintaining training stability.

\subsection{Updations on the Conversion Process }
Our study enhances the quant-ANN-to-SNN conversion process by introducing several innovative updates to existing frameworks.

Previous approaches address noise, including occasional noise, only after completing the ANN-to-SNN conversion. In contrast, our method integrates the modeling of occasional noise directly within the ANN quantization training phase. By introducing noise during both the forward and backward passes of training, our model learns to adapt to it. This learning-based method enables seamless end-to-end conversion from ANNs to SNNs, removing the necessity to alter the spiking neuron model or manage control flow during SNN inference.

A common challenge in leveraging pre-trained ANN models for ANN-to-SNN conversion is the discrepancy in pooling layer preferences -- SNNs typically use average-pooling, while ANNs favor max-pooling. This difference has limited the ability of previous methods to effectively utilize pre-trained ANN models with max-pooling layers. Our approach overcomes this limitation by being agnostic to the type of pooling layer. During ANN quantization training, we initialize with a full-precision pre-trained ANN model that uses max-pooling. We then substitute max-pooling with average-pooling and proceed with activation quantization training. Our experimental results confirm that this substitution does not hinder training convergence.

\subsection{Connections to Existing Theoretical Work}
Our study is related to some theoretical research on quant-ANN-to-SNN conversion. This line of work primarily focuses on providing a theoretical guarantee of minimizing the disparity between the response curve of spiking neurons and the activation quantization function used in ANNs. An optimal conversion from a quantized ANN to an SNN is achieved when this disparity is zero, resulting in a conversion error of zero. However, these approaches also face certain challenges. Specifically, the behavior of occasional noise in SNNs becomes visible only during actual SNN runtime, making it difficult to precisely model this noise within the ANN quantization function. Consequently, most current theoretical approaches exclude this noise term in their proofs, offering only an approximate estimation of the conversion error. Recent studies have proposed methods to manage occasional noise during SNN runtime \cite{hao2023reducing, hao2023bridging}; however, these typically require additional simulation stages, which increase the complexity of the SNN. Furthermore, these studies often assume equivalence between activation-quantized ANNs and SNNs under the condition that the time step $T=p$, where $p$ is the quantization upper bound as defined in Equation \ref{eq3}. A theoretically optimal conversion is only feasible at this one specific time step. Providing valid theoretical proofs becomes significantly more complex when $T>p$.

\begin{figure}
\centering
\includegraphics[width=0.465\textwidth]{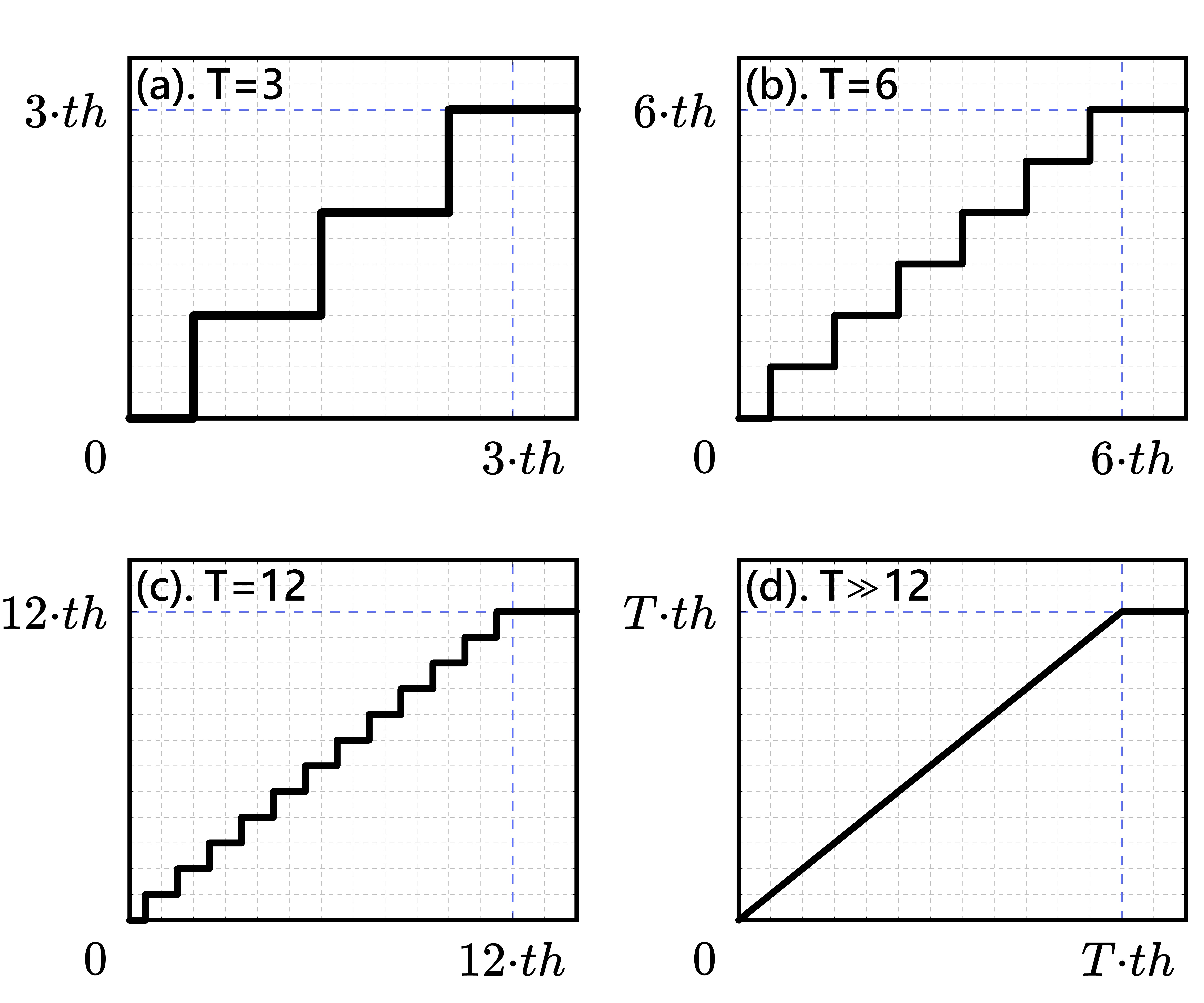} 

\caption{The response curve of spiking neurons for time steps $T=3$, $T=6$, $T=12$, and for time steps much larger than 12. As $T$ increases, the response smoothens and converges to a rectified-ReLU function.}
\label{fig2}
\end{figure}

Our proposed Noise Adaptor algorithm contributes to this line of work in the following ways: First, we introduced a practical and simple noise term to enable modeling occasional noise in the ANN quantization equation. Our experimental results demonstrate the effectiveness of this approach in reducing conversion error and enhancing SNN performance. Second, our noise model scales effectively to very deep neural network architectures. In the experimental section, we present successful conversion results for quantized ResNet-101 and ResNet-152 to low-latency SNNs. Furthermore, Noise Adaptor enables the ANN model to implicitly optimize for SNNs in scenarios where the simulation time $T$ significantly exceeds $p$. Below, We provide visualizations, detailed explanations, and the corresponding mathematical formulation of this contribution.

We visualize the response curve of spiking neurons for different time steps, specifically \(T=3\), \(T=6\), \(T=12\), and \(T \gg 12\), in Figure \ref{fig2}. Here, \(th\) denotes the firing threshold of the spiking neuron, and we assume the input to the neuron is a constant value. The x-axis represents the accumulated input values over time $T$, while the y-axis indicates the output spike counts. To maintain consistency in amplitude between the input and output of the spiking neuron, we multiply the output spike counts by the firing threshold. As shown in Figure \ref{fig2}(a), when \(T=3\), the maximum output is \(3 \cdot th\), which corresponds to generating 3 spikes in response to the input of \(3 \cdot th\). As the time step \(T\) increases to $6$ and $12$, more discrete values can be observed between 0 and the maximum output, resulting in a smoother curve. When \(T\) becomes sufficiently large, the response curve of the spiking neuron converges to that of a rectified-ReLU function, as depicted in Figure \ref{fig2}(d). The activation function in ANN that corresponds to this SNN response curve can be derived using the quant-ANN-to-SNN conversion equations. We first identify the input that yields the maximum output in a single time step, which should be \(th\). According to Equation \eqref{eq4}, \(th\) is calculated by the product of the activation quantization parameters \(s\) and \(p\) in the ANN. Therefore, the activation function we aim to optimize in the ANN should be the rectified-ReLU function, with its inflection point set at \(ps\). This can be mathematically expressed as:
\begin{equation}\label{eq9}
\hat{v} = \text{clip}(v, 0, sp). 
\end{equation}

Our proposed Noise Adaptor can implicitly optimize this function during ANN quantization training. Specifically, after incorporating the Noise Adaptor, the expected mean of the applied activation quantization function in ANN is given by:
\begin{equation}\label{eq10}
	\mathbb{E}(\hat{v}^l) = \text{clip}(v^l, 0, {s^l}p).
\end{equation}

\noindent This expression is consistent with the rectified-ReLU function presented in Equation \eqref{eq9}. The proof of this equation is provided in the Supplementary Materials. These findings suggest that Noise Adaptor can facilitate the implicit optimization of the SNN response curve, particularly at higher time step $T$. 

\begin{table}[htb!]
    \centering

     \resizebox{\columnwidth}{!}
     {
    \begin{tabular}{ccccccc}
    \toprule
    \multirow{2}{*}{\textbf{CIFAR-10}} & \multirow{2}{*}{\textbf{ANN}} & \multicolumn{5}{c}{$\pmb{T}$}\\
    \cmidrule(lr){3-7}

   & &T=1 & T=2 & T=4 & T=8 & T=16  \\
    \midrule
    
ResNet-18 w/o NA	&\textbf{95.50}&91.62& 93.65 &95.18& 96.08 & 96.38 \\
ResNet-18 w/ NA  &94.97&\textbf{93.71}& \textbf{95.26} &\textbf{95.95}& \textbf{96.61} & \textbf{96.72}	\\

\midrule[\heavyrulewidth]

    \multirow{2}{*}{\textbf{ImageNet}} & \multirow{2}{*}{\textbf{ANN}} & \multicolumn{5}{c}{$\pmb{T}$}\\
    \cmidrule(lr){3-7}

   & & T=8 & T=16 & T=32 & T=64 & T=128  \\
    \midrule
    
ResNet-50 w/o NA&\textbf{75.96}&33.94&56.42& 69.33 &74.95 & 76.84 \\
ResNet-50 w/ NA  &75.67&\textbf{37.08}&\textbf{58.76}&\textbf{71.50} &\textbf{75.69}& \textbf{76.95}	\\

      \bottomrule
    \end{tabular} 
     } 
    \vspace{-1mm}
    \caption{Comparisons of quant-ANN-to-SNN conversion with and without Noise Adaptor (NA). We benchmark ANN accuracy after training with activation quantization, as well as SNN accuracy at different time steps $T$.}
    \vspace{-1mm}
    \label{tab1}
\end{table}

%% file: sec/5.tex
\section{Experimental Results}

\subsection{Experimental Setup}

\begin{table*}[tb]
\centering
\footnotesize
\resizebox{\textwidth}{!}
    {
    \begin{tabular}{ccccccccccccc}
    \toprule
     \textbf{Methods}&\textbf{Datasets}&  \textbf{Architecture}& \textbf{ANN}  & \textbf{T=1} & \textbf{T=2} & \textbf{T=4} & \textbf{T=8} & \textbf{T=16} & \textbf{T=32} & \textbf{T=64} & \textbf{T=128}& \textbf{T=256}\\
    \midrule

    QCFS\cite{bu2021optimal} & \multirow{4}{*}{CIFAR-10} & \multirow{4}{*}{ResNet-18}  & \textbf{96.04\%} & - & 75.44\% & 90.43\% & 94.82\% & 95.92\%& 96.08\% &96.06\% & - &96.06\%\\
    SlipReLU~\cite{jiang2023unified} &  &   & 94.61\% & 93.11\% & 93.97\% & 94.59\% & 94.92\% & 95.18\% & 95.07\% & 94.81\% & - &94.67\%\\

    TTRBR~\cite{meng2022training} &  &   & 95.27\% & - & - & - & -& 93.99\% & 94.77\% & 95.04\% & 95.18\% & -\\

    \textbf{Noise Adaptor} &                     &   &94.97\%&\textbf{93.71\%}& \textbf{95.26\%} &\textbf{95.95\%}& \textbf{96.61\%} & \textbf{96.72\%} &\textbf{96.74\%} &\textbf{96.72\%} &\textbf{96.69\%} &\textbf{96.68\%}\\

    \midrule

    QCFS\cite{bu2021optimal} & \multirow{7}{*}{ImageNet} & ResNet-34 & 74.32\% & - & - & - & - & 59.35\%& 69.37\% &72.35\% & 73.15\% &73.39\%\\
    
    SlipReLU~\cite{jiang2023unified} &  & ResNet-34 & 75.08\% & -& - & - & - & 43.76\% & 66.61\% & 72.71\% & 74.01\% &-\\

    TTRBR~\cite{meng2022training} &  & ResNet-34  & 74.24\% & - & - & - & -& - & - & - & - & 74.18\%\\
    
    TTRBR~\cite{meng2022training} &  & ResNet-50  & 76.02\% & - & - & - & -& - & - & - & - & 75.04\%\\

    TTRBR~\cite{meng2022training} &  & ResNet-101  & 76.82\% & - & - & - & -& - & - & - & - & 75.72\%\\
    \textbf{Noise Adaptor} &                    & ResNet-34 &71.26\%&\textbf{25.73\%}& \textbf{36,31\%} &\textbf{51.02\%}& \textbf{63.17\%} 
    & \textbf{69.71\%}  &\textbf{72.32\%} &73.03\% &73.07\%&73.02\%\\

    \textbf{Noise Adaptor} &                    & ResNet-50 &75.79\%&9.22\%& 12.00\% &19.33\%& 37.08\% 
    & 58.76\%  &71.50\% &\textbf{75.69\%} &\textbf{76.95\%}&77.07\%\\

    \textbf{Noise Adaptor} &                    & ResNet-101 &78.61\%&6.06\%& 6.87\% &9.23\%& 17.77\% 
    & 37.04\%  &58.96\% &71.89\% &76.81\%&\textbf{78.44\%}\\

    \textbf{Noise Adaptor} &                    & ResNet-152 &\textbf{79.29}\%&4.44\%& 5.80\% &8.67\%& 14.29\% 
    & 24.02\%  &40.87\% &59.58\% &70.76\%&76.04\%\\

    \bottomrule
    \end{tabular}}
    \vspace{0.2mm}
        \caption{Comparisons with other quant-ANN-to-SNN conversion algorithms. The comparisons are limited within algorithms without noise correction during SNN inference.}
        \vspace{-0.2cm}
    \label{tab3}
\end{table*}

We validate the effectiveness of our method on CIFAR-10 \cite{krizhevsky2009learning} and ImageNet \cite{deng2009imagenet}, utilizing the ResNet network architecture. These datasets are chosen for evaluations on image classification tasks, facilitating direct comparisons with results reported in other papers. We implement quantization-aware training (QAT) based on pre-trained full-precision ResNet models to quantize the activation. Before training, we replace max-pooling with average-pooling and substitute the ReLU activation function with the quantization function described in Equation \eqref{eq5}, while keeping the network weights in full precision. The quantization scale factor $s^l$ is initialized in a manner that minimizes the $\mathcal{L}_2$ difference between the activation of full-precision model and those of the quantized model. This initialization is performed sequentially, layer-by-layer, until $s^l$ is initialized across all layers. To prevent the gradient of the quantization scale factor $s^l$ from becoming excessively large, we apply a gradient scaling factor to regulate its amplitudes, thereby aiding the model in achieving better convergence. Training is carried out using stochastic gradient descent (SGD) with a cross-entropy loss function and a cosine learning rate scheduler. The initial learning rate is set to $0.05$, with a weight decay of $5 \times 10^{-4}$ for CIFAR-10 and $2.5 \times 10^{-5}$ for ImageNet. The models are trained for 400 epochs on CIFAR-10 and 600 epochs on ImageNet. The quantization upper bound $p$ is set to $2$ for CIFAR-10 and $3$ for ImageNet, unless specified otherwise.

After training, we transfer all weights and biases from the source ANN to the converted SNN, setting the threshold $th^l$ as outlined in Equation \eqref{eq4}. To address systematic errors associated with the round-to-nearest method used during ANN quantization, we pre-charge the membrane potential of the spiking neurons in all hidden layers to $0.5th^l$ at the first time step during SNN simulation. In these simulations, inputs are analog-coded \cite{rueckauer2017conversion} with a time resolution $\Delta t$ of $1\,$ms, and the adopted spiking neuronal model is described by Equation \eqref{eq1} and Equation \eqref{eq2}.

\subsection{Latency and Accuracy Improvements by Noise Adaptor}
Table \ref{tab1} presents the impact of Noise Adaptor on SNN performance across the CIFAR-10 and ImageNet datasets. The baseline for comparison is the ANN model trained without Noise Adaptor during quantization-aware training, as outlined in Equation \eqref{eq3}. 

On CIFAR-10, the results indicate that Noise Adaptor significantly enhances SNN performance. The accuracy consistently improves as the simulation time step $T$ increases, reaching a peak of 96.72\% when $T=16$. A similar trend is observed on ImageNet, where the SNN model utilizing the Noise Adaptor demonstrates substantial improvements across all $T$ values compared to the baseline, achieving a peak accuracy of $76.95\%$ at $T=128$. These findings demonstrate the effectiveness of Noise Adaptor in boosting SNN accuracy across varying datasets and time steps.

Interestingly, the data also reveals that a high ANN accuracy is not a prerequisite for attaining a high SNN accuracy. In our CIFAR-10 experiments, we observed a slight reduction in ANN accuracy when Noise Adaptor was applied. Specifically, models with Noise Adaptor achieved an ANN accuracy of $94.97\%$, which is $0.53\%$ lower compared to the model without it. Despite this reduction, the inclusion of Noise Adaptor consistently led to superior SNN inference accuracy across all time steps. A similar trend was observed in our ImageNet experiments, where Noise Adaptor reduced ANN accuracy but enhanced SNN performance. We attribute this decrease in ANN accuracy to the stochasticity introduced by Noise Adaptor, which negatively impacts ANN performance. However, this stochasticity appears to facilitate the Quant-ANN-to-SNN conversion process, as evidenced by the improved SNN accuracy after applying Noise Adaptor.

\subsection{Comparision to Other Quant-ANN-to-SNN Conversion Methods}
In this section, we compare our proposed method with existing quant-ANN-to-SNN conversion techniques on CIFAR-10 and ImageNet datasets. Specifically, we compare with QCFS \cite{ding2021optimal}, SlipReLU \cite{jiang2023unified}, and TTRBR \cite{meng2022training}. Similar to our approach, these methods do not incorporate runtime noise correction strategies, such as negative spikes or two-stage simulations during the inference of SNNs, ensuring a fair basis for comparison.

The results are presented in Table \ref{tab3}. Our method consistently achieves the highest SNN accuracy across all time steps. On CIFAR-10, our method attained accuracies of $95.26\%$ at $T=2$ and $95.95\% $ at $T=4$, significantly outperforming other reported results. The peak accuracy of our method reaches $96.74\%$ at $T=32$, surpassing the previous best-reported accuracy of $96.08\%$, which was achieved by QCFS at $T=32$. Similarly, on ImageNet, our method demonstrates a clear performance advantage. While QCFS and SlipReLU achieve SNN accuracies above $ 70\%$ only at $T=64$, our method surpasses this threshold in just $32$ time steps.

\begin{table}[htb!]
    \centering

     \resizebox{\columnwidth}{!}
     {
    \begin{tabular}{cccccccc}
    \toprule
    \multirow{2}{*}{\textbf{Noise Correction Method}} & \multicolumn{6}{c}{\textbf{SNN accuracy on CIFAR-10}}\\
    \cmidrule(lr){2-7}

&T=1 & T=2 & T=4 & T=8 & T=16 & T $\geqslant$ 256  \\
    \midrule
None (Noise Adaptor)	&93.71& 95.26 &95.95& 96.61 & 96.72 & 96.68  \\
Negative Spikes	&93.71& 95.87 &96.55& 96.70 & 96.71& 96.73 \\

Offset Spikes (lightweight)&95.82& 96.25 &96.33& 96.66 & 96.67& 96.66 \\
    \midrule
    \multirow{2}{*}{\textbf{Noise Correction Method}} & \multicolumn{6}{c}{\textbf{SNN Accuracy on ImageNet}}\\
    \cmidrule(lr){2-7}

&T=1 & T=2 & T=4 & T=8 & T=16 & T $\geqslant$ 256  \\
    \midrule
None (Noise Adaptor)	&9.22& 12.00 &19.33& 37.08 & 58.76 & 77.07  \\
Negative Spikes	&9.22& 55.32 &71.79& 74.66 & 75.93& 75.81 \\

Offset Spikes (lightweight)&64.33& 65.56 &68.53& 70.66 & 72.35& 75.56 \\
      \bottomrule
    \end{tabular} 
     } 
    \vspace{-1mm}
    \caption{The compatibility of Noise Adaptor with noise corrections techniques applied during SNN inference. The network architecture is ResNet-18 on CIFAR-10 and ResNet-50 on ImageNet. Accuracy is expressed as percentage (\%).}
    \vspace{-1mm}
    \label{tab4}
\end{table}

\subsection{Compatability to Other Noise Correction Techniques}
While our proposed method is designed to handle occasional noise during ANN training, it can also be effectively integrated with other noise correction techniques during SNN runtime to enhance overall SNN performance.

In Table \ref{tab4}, we examine the compatibility of our proposed Noise Adaptor with various noise correction methods, including negative spikes \cite{li2022quantization} and the two-stage offset spikes correction \cite{hao2023bridging}, applied to the CIFAR-10 and ImageNet datasets. The results show that integrating our approach with these noise correction techniques enhances SNN performance, especially at early time steps. This suggests that for scenarios where ultra-fast SNN inference is critical, applying runtime noise correction techniques can be beneficial.

However, it is important to recognize that these noise correction techniques may also introduce additional computational overhead during inference. We also observe that, as the number of time steps increases, networks employing only the Noise Adaptor gradually reach the accuracy levels of those that use noise correction methods.

\subsection{Scalability Study on Network Depth}
In this section, we assess the scalability of the Noise Adaptor. Table \ref{tab3} illustrates the performance of various ResNet architectures on the ImageNet dataset, each integrated with the Noise Adaptor during ANN activation quantization training, with a quantization upper bound of $p=3$. The results indicate a clear trend: as the network depth increases, there is a consistent improvement in ANN accuracy, with the deepest network, ResNet-152, achieving the highest accuracy at 79.29\%. Similarly, SNN accuracy also improves with increasing network depth, highlighting the Noise Adaptor's effectiveness in enabling successful ANN-to-SNN conversions even for very deep networks.

However, ResNet-152 requires more than 256 time steps to approach the target ANN accuracy, suggesting that such deep neural networks may not be practical for SNN implementation due to their high computational demand. Notably, ResNet-34 achieves an accuracy of 72.32\% in just 32 time steps, while ResNet-50 reaches a competitive accuracy of 75.69\% in 64 time steps. These findings suggest that deeper networks generally necessitate more time steps to achieve the desired accuracy.

\begin{figure}
\centering
\includegraphics[width=0.35\textwidth]{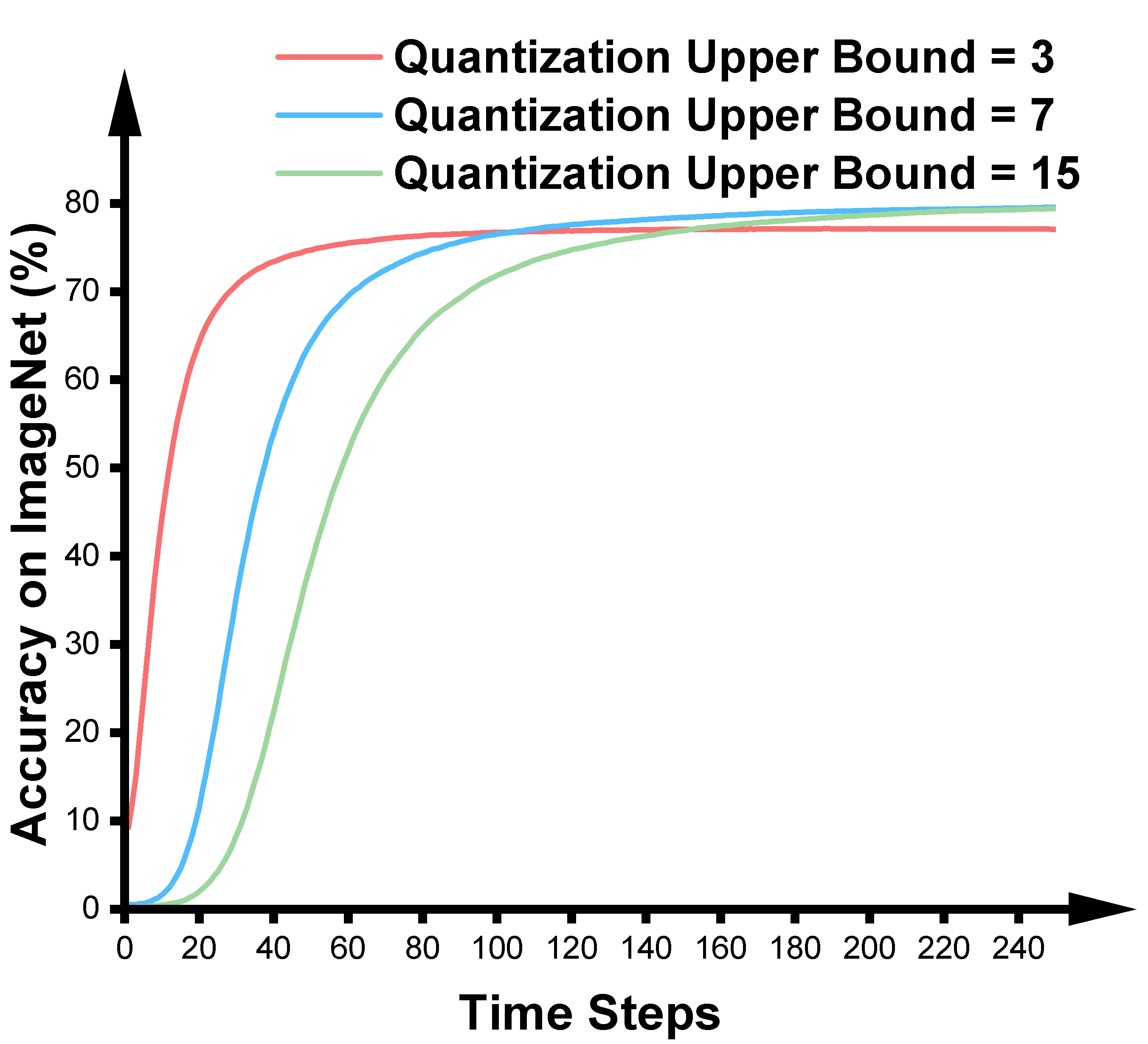} 

\caption{SNN accuracy over time steps for ResNet-50 with quantization upper bounds $p$ set at 3, 7, and 15. Results show that a small $p$ initially improves accuracy more rapidly, but a higher $p$ may surpass it over time.}
\label{fig3}
\end{figure}

\subsection{Quantitative Analysis of Quantization States}

In the context of low-bit ANN training and quant-ANN-to-SNN conversion, the quantization upper bound \( p \) is a crucial hyperparameter that significantly impacts both the accuracy of the ANN and the latency and target accuracy of the resulting SNN. Figure \ref{fig3} illustrates the SNN accuracy over time steps for ResNet-50, with quantization upper bounds set at 3, 7, and 15, respectively. The figure shows that a quantization upper bound of \( p = 3 \) leads to the fastest accuracy improvements in the initial time steps. However, by around 100 time steps, the accuracy for \( p = 7 \) surpasses that of \( p = 3 \). Interestingly, the accuracy for \( p = 15 \) consistently remains lower than that of \( p = 7 \) throughout the observed time window of 256 steps. These findings underscore the importance of carefully selecting the quantization upper bound \( p \) when constructing SNNs.

%% file: sec/6.tex
\section{Conclusion}
The Noise Adaptor contributes to research on converting low-bit ANNs into low-latency SNNs, effectively addressing key challenges in the field. By incorporating noise during ANN training, our approach enhances the model's awareness of spike dynamics, leading to improved SNN accuracy. The proposed method is efficient and scalable, successfully demonstrated on several deep architectures and showing competitive performance.

%% file: sec/suppl.tex
\clearpage
\onecolumn
\setcounter{page}{1}

\maketitlesupplementary

\theoremstyle{definition}
\newtheorem{theorem}{Theorem}

\section{Proof of Equation 10: Expected Mean of Activation Quantization Function in ANN}

In this section, we present a proof for Equation 10, demonstrating that the expected mean of ANN activation quantization function, after incorporating a Noise Adaptor, aligns with the rectified ReLU function.

\begin{theorem}
Let \( \hat{v}^l \) be the output of the function 
\begin{align}
\hat{v}^l = s^l \left\lfloor \text{clip}\left( \frac{v^l}{s^l} + \epsilon, 0, p \right) \right\rceil,
\end{align}
where \( \epsilon \) is a uniformly distributed random variable over the interval \( (-0.5, 0.5) \), i.e., \( \epsilon \sim \mathcal{U}(-0.5, 0.5) \). The mean output \( \mathbb{E}[\hat{v}^l] \) is given by:
\begin{align}
\mathbb{E}[\hat{v}^l] = \text{clip}(v^l, 0, s^l p).
\end{align}
\end{theorem}

\begin{proof}
Define \( u = \frac{v^l}{s^l} \). The function then becomes:
\begin{align}
\hat{v}^l = s^l \left\lfloor \text{clip}(u + \epsilon, 0, p) \right\rceil.
\end{align}
We will consider three cases based on the value of \( u \):

\textbf{Case 1:} \( 0 \leq u < p \).

After rounding, \( \left\lfloor \text{clip}(u + \epsilon, 0, p) \right\rceil \) will be rounded up or down depending on whether its value is higher or lower than the transition point \( \lfloor u \rfloor + 0.5 \), as given below:
\begin{align}
\left\lfloor \text{clip}(u + \epsilon, 0, p) \right\rceil = 
\begin{cases} 
\lfloor u \rfloor & \text{if } u + \epsilon < \lfloor u \rfloor + 0.5, \\
\lceil u \rceil & \text{if } u + \epsilon \geq \lfloor u \rfloor + 0.5.
\end{cases}
\end{align}
Let \( w = u - \lfloor u \rfloor \), which is the fractional part of \( u \). The condition \( u + \epsilon \geq \lfloor u \rfloor + 0.5 \) can be rewritten as:
\begin{align}
u + \epsilon &\geq \lfloor u \rfloor + 0.5 \\
w + \lfloor u \rfloor + \epsilon &\geq \lfloor u \rfloor + 0.5 \\
w + \epsilon &\geq 0.5 \\
\epsilon &\geq 0.5 - w.
\end{align}
Since \( \epsilon \sim \mathcal{U}(-0.5, 0.5) \), it follows that \( \epsilon \) is uniformly distributed between \(-0.5\) and \(0.5\). The probability that \( \epsilon \geq 0.5 - w \) is:
\begin{align}
\mathbb{P}(\epsilon \geq 0.5 - w) = \mathbb{P}(\mathcal{U}(0, 1) \geq 1 - w) = w.
\end{align}
Thus, the probability of rounding up is \( w \), and the probability of rounding down is \( 1 - w \). Therefore:
\begin{align}
\left\lfloor \text{clip}(u + \epsilon, 0, p) \right\rceil = 
\begin{cases} 
\lfloor u \rfloor & \text{with probability } 1 - w, \\
\lceil u \rceil & \text{with probability } w.
\end{cases}
\end{align}

To find the expected value, we use the expectation formula:
\begin{align}
\mathbb{E}\left[\left\lfloor \text{clip}(u + \epsilon, 0, p) \right\rceil\right] &= \lfloor u \rfloor \cdot (1 - w) + \lceil u \rceil \cdot w \\
&= \lfloor u \rfloor \cdot (1 - w) + (\lfloor u \rfloor + 1) \cdot w \\
&= \lfloor u \rfloor + w.
\end{align}
Since \( u = w + \lfloor u \rfloor \), it follows that:
\begin{align}
\mathbb{E}\left[\left\lfloor \text{clip}(u + \epsilon, 0, p) \right\rceil\right] = u.
\end{align}
Hence:
\begin{align}
\mathbb{E}[\hat{v}^l] &= s^l \mathbb{E}\left[\left\lfloor \text{clip}(u + \epsilon, 0, p) \right\rceil\right] \\
&= s^l \cdot u \\
&= v^l.
\end{align}

\textbf{Case 2:} \( u \geq p \).

For \( u \geq p \), the smallest value of \( u + \epsilon \) occurs when \( \epsilon = -0.5 \), which gives \( u + \epsilon \geq p - 0.5 \). After clipping and rounding, the value \( p - 0.5 \) will be rounded to \( p \). Hence, for all \( u \geq p \):
\begin{align}
\left\lfloor \text{clip}(u + \epsilon, 0, p) \right\rceil = p.
\end{align}
Thus:
\begin{align}
\mathbb{E}[\hat{v}^l] &= s^l \left\lfloor \text{clip}(u + \epsilon, 0, p) \right\rceil \\
&= s^l \cdot p.
\end{align}

\textbf{Case 3:} \( u < 0 \).

For \( u < 0 \), \( u + \epsilon \) is always less than 0.5. Hence, rounding forces the output to 0. Thus:
\begin{align}
\mathbb{E}[\hat{v}^l] = 0.
\end{align}

Combining these results, the expected output is:
\begin{align}
\mathbb{E}[\hat{v}^l] = 
\begin{cases} 
v^l & \text{if } 0 \leq \frac{v^l}{s^l} < p, \\
s^l \cdot p & \text{if } \frac{v^l}{s^l} \geq p, \\
0 & \text{if } \frac{v^l}{s^l} < 0.
\end{cases}
\end{align}
Thus, we have:
\begin{align}
\mathbb{E}[\hat{v}^l] = \text{clip}(v^l, 0, s^l p).
\end{align}
\end{proof}

\section{Pseudo Code for Implementing Noise Adaptor}

\begin{algorithm}
\caption{ANN Activation Quantization Function with Noise Adaptor}
\label{alg:algorithm}
\textbf{Input}: Input tensor $x$, scale factor $scale$, quantization lower bound $0$, quantization upper bound $p$ \\
\textbf{Output}: Quantized tensor $\hat{x}$ \\
\begin{algorithmic}[1] 

\State \textbf{Forward Pass:}
\State Compute $x\_scale = \frac{x}{scale}$
\State Add random noise: $x\_scale \mathrel{+}= \epsilon$
\State Clip $x\_scale$ to bounds: $x\_clip = \text{clip}(x\_scale, 0, p)$
\State Round the clipped values: $x\_round = \text{round}(x\_clip)$
\State Restore scaled values: $\hat{x} = x\_round \times scale$
\State \textbf{return} $\hat{x}$

\Statex 

\State \textbf{Backward Pass:}
\State Compute internal flag: $internal\_flag = (x\_clip > 0) \oplus (x\_clip \geq p)$
\State Compute gradient for activation: $grad\_activation = grad\_output \times internal\_flag$
\State Compute first part of scale gradient: $grad\_one = x\_clip \times internal\_flag$
\State Compute second part of scale gradient: $grad\_two = \text{round}(x\_clip)$
\State Compute element-wise scale gradient: $grad\_scale\_elem = grad\_two - grad\_one$
\State Compute final scale gradient: $grad\_scale = \frac{(grad\_scale\_elem \times grad\_output).sum()}{\sqrt{\text{len}(x\_clip) \times p}}$
\State \textbf{return} $grad\_activation$, $grad\_scale$, 

$grad\_lowerbound = 0$, $grad\_upperbound = 0$
\end{algorithmic}
\end{algorithm}